\def\year{2021}\relax
\documentclass[letterpaper]{article} % DO NOT CHANGE THIS
\usepackage{aaai21}  % DO NOT CHANGE THIS
\usepackage{times}  % DO NOT CHANGE THIS
\usepackage{helvet} % DO NOT CHANGE THIS
\usepackage{courier}  % DO NOT CHANGE THIS
\usepackage[hyphens]{url}  % DO NOT CHANGE THIS
\usepackage{graphicx} % DO NOT CHANGE THIS
\usepackage{natbib}  % DO NOT CHANGE THIS AND DO NOT ADD ANY OPTIONS TO IT
\usepackage{caption} % DO NOT CHANGE THIS AND DO NOT ADD ANY OPTIONS TO IT
\usepackage{dsfont}
\usepackage{MnSymbol}
\usepackage{mathtools}
\usepackage{amsmath,amsthm}
\usepackage{bm}
\usepackage{amsfonts}       % blackboard math symbols
\usepackage{nicefrac}       % compact symbols for 1/2, etc.
\usepackage{booktabs}       % professional-quality tables
\usepackage{graphicx}
\usepackage{subcaption}

\usepackage{thmtools}
\usepackage{thm-restate}

\newcommand{\E}{\mathbb{E}}
\newcommand{\R}{\mathbb{R}}
\newcommand{\indypsi}{\mathds{1}_{\{y \in\Psi(x)\}}}

\newcommand{\indyphi}{\mathds{1}_{\{y \in\Phi(x)\}}}
\newcommand{\indynotphi}{\mathds{1}_{\{y \in \Phi^0(x)\}}}
\newcommand{\Prob}{\mathcal{P}}
\newcommand{\lowr}{\nabla}
\newcommand{\highr}{\Delta}

\DeclarePairedDelimiterX{\infdivx}[2]{(}{)}{%
  #1\;\delimsize\|\;#2%
}
\DeclarePairedDelimiter\abs{\lvert}{\rvert}

\newcommand{\chidiv}{\Delta_{\chi^2}\infdivx}
\newcommand{\MSE}[1]{\textrm{MSE}\left[#1\right]}
\newcommand{\VsIPS}[1]{\hat{V}^\Phi_\textrm{sIS}{(#1)}}
\newcommand{\VIPS}[1]{\hat{V}_\textrm{IS}{(#1)}}

\newtheorem{asp}{Assumption}

\title{Learning from eXtreme Bandit Feedback}
\author {
    Romain Lopez\textsuperscript{\rm 1},
    Inderjit S. Dhillon\textsuperscript{\rm 2, \rm 3},
    Michael I. Jordan\textsuperscript{\rm 1, \rm 2} \\
}
\affiliations {
    \textsuperscript{\rm 1} Department of Electrical Engineering and Computer Sciences, University of California, Berkeley\\
    \{romain\_lopez, jordan\}@cs.berkeley.edu\\
    \textsuperscript{\rm 2} Amazon.com \\
    \textsuperscript{\rm 3} Department of Computer Science, The University of Texas at Austin\\
    inderjit@cs.utexas.edu\\

}
\begin{document}

\maketitle

\begin{abstract}
We study the problem of batch learning from bandit feedback in the setting of extremely large action spaces. Learning from extreme bandit feedback is ubiquitous in recommendation systems, in which billions of decisions are made over sets consisting of millions of choices in a single day, yielding massive observational data. In these large-scale real-world applications, supervised learning frameworks such as eXtreme Multi-label Classification (XMC) are widely used despite the fact that they incur significant biases due to the mismatch between bandit feedback and supervised labels. Such biases can be mitigated by importance sampling techniques, but these techniques suffer from impractical variance when dealing with a large number of actions. 
In this paper, we introduce a \emph{selective importance sampling estimator} (sIS) that operates in a significantly more favorable bias-variance regime. The sIS estimator is obtained by performing importance sampling on the conditional expectation of the reward with respect to a small subset of actions for each instance (a form of Rao-Blackwellization). We employ this estimator in a novel algorithmic procedure---named Policy Optimization for eXtreme Models (POXM)---for learning from bandit feedback on XMC tasks. In POXM, the selected actions for the sIS estimator are the top-$p$ actions of the logging policy, where $p$ is adjusted from the data and is significantly smaller than the size of the action space. 
We use a supervised-to-bandit conversion on three XMC datasets to benchmark our POXM method against three competing methods: BanditNet, a previously applied partial matching pruning strategy, and a supervised learning baseline. Whereas BanditNet sometimes improves marginally over the logging policy, our experiments show that POXM systematically and significantly improves over all baselines.
\end{abstract}
\section{Introduction}

In the classical supervised learning paradigm, it is assumed that every data point is accompanied by a label.  Such labels provide a very strong notion of feedback,
where the learner is able to assess not only the loss associated with the action that they have chosen but can also assess losses of actions that they did not choose.  A useful weakening of this paradigm involves considering so-called ``bandit feedback,'' where the training data simply provides evaluations of selected actions without delineating the correct action.  Bandit feedback is often viewed as the province of reinforcement learning, but it is also possible to combine bandit feedback with supervised learning by considering a batch setting in which each data point is accompanied by an evaluation and there is no temporal component.  This is the Batch Learning from Bandit Feedback (BLBF) problem~\cite{Swaminathan2015CF}.

Of particular interest is the off-policy setting where the training data is provided by a \emph{logging policy}, which differs from the learner's policy and differs from the optimal policy.  Such problems arise in many real-world problems, including supply chains, online markets, and recommendation systems~\cite{8777645}, where abundant data is available in a logged format but not in a classical supervised learning format.

Another difficulty with the classical notion of a ``label'' is that real-world problems often involve huge action spaces.  This is the case, for example, in real-world recommendation systems where there may be billions of products and hundreds of millions of consumers. Not only is the cardinality of the action space challenging both from a computational point of view and a statistical point of view, but even the semantics of the labels can become obscure---it can be difficult to place an item conceptually in one and only category. Such challenges have motivated the development of eXtreme multi-label classification (XMC)
and eXtreme Regression (XR)~\cite{Bhatia16} methods, which focus on computational scalability issues and target settings involving millions of labels.  These methods have had real-world applications in domains such as e-commerce~\cite{agrawal2013multi} and dynamic search advertising~\cite{prabhu2018parabel,prabhu2020extreme}.

We assert that the issues of bandit feedback and extreme-scale action spaces are related.  Indeed, it is when action spaces are large that it is particularly likely that feedback will only be partial.  Moreover, large action spaces tend to support multiple tasks and
grow in size and scope over time, making it likely that available data will be in the form of a logging policy and not a single target input-output mapping.

We also note that the standard methodology for accommodating the difference between the logging policy and an optimal policy needs to be considered carefully in the setting of large action spaces.  Indeed, the standard methodology is some form of importance sampling~\cite{Swaminathan2015CF}, and importance sampling estimators can run aground when their variance is too high (see, e.g.,~\citet{LefortierSGJR16}).  Such variance is likely to be particularly virulent in large action spaces. Some examples of the XMC framework do treat labels as subject to random variation~\cite{jain2016extreme}, but only with the goal of improving the prediction of rare labels; they do not tackle the broader problem of learning from logging policies in extreme-scale action spaces.  It is precisely this broader problem that is our focus in the current paper.

The literature on offline policy learning in Reinforcement Learning (RL) has also been concerned with correcting for implicit feedback bias (see, e.g.,~\citet{degris2012off}).  This line of work differs from ours, however, in that the focus in RL is not on extremely-large action spaces, and RL is often based on simulators rather than logging policies~\cite{chen2019gan,bai2019}. Closest to our work is the work of \citet{chen2019topk}, who propose to use offline policy gradients on a large action space (millions of items). Their method relies, however, on a proprietary action embedding, unavailable to us. 

After a brief overview of BLBF and XMC %(Section~\ref{sec:background})
, we present a new form of BLBF that blends bandit feedback with multi-label classification% (Section~\ref{sec:bandit-mulilabel})
. We introduce a novel assumption, specific to the XMC setting, in which most actions are irrelevant (i.e., incur a null reward) for a particular instance. This motivates a Rao-Blackwellized~\cite{casella1996rao} estimator of the policy value for which only a small set of relevant actions per instance are considered. We refer to this approach as \emph{selective importance sampling} (sIPS).  We provide a theoretical analysis of the bias-variance tradeoff of the sIPS estimator compared to naive importance sampling% (Section~\ref{sec:learning_XBF})
. In practice, the selected actions for the sIPS estimator are the top-$p$ actions from the logging policy, where $p$ can be adjusted from the data. We derive a novel learning method based on the sIPS estimator, which we refer to as \emph{Policy Optimizer for eXtreme Models} (POXM). Finally, we propose a modification of a state-of-the-art neural XMC method AttentionXML~\cite{you2019attentionxml} to learn from bandit feedback. Using a supervised-learning-to-bandit conversion~\cite{Dudik2011}, we benchmark POXM against BanditNet~\cite{Joachims2018}, a partial matching scheme from~\citet{wang2016beyond} and a supervised learning baseline on three XMC datasets (EUR-Lex, Wiki10-31K and Amazon-670K)~\cite{Bhatia16}. We show that naive application of the state-of-the-art method BanditNet~\cite{Joachims2018} sometimes improves over the logging policy, but only marginally. Conversely, POXM provides substantial improvement over the logging policy as well as supervised learning baselines.% (Section~\ref{sec:experiments}).

\section{Background}
\label{sec:background}

\subsection{eXtreme Multi-label Classification (XMC)}
Multi-label classification aims at assigning a relevant subset $\bm{Y} \subset [L]$ to an instance $x$, where $[L]:=\{1,\ldots,L\}$ denotes the set of $L$ possible labels. XMC is a specific case of multi-label classification in which we further assume that all $\bm{Y}$ are small subsets of a massive collection (i.e., generally $|\bm{Y}|/{L} < 0.01$). Naive one-versus-all approaches to multi-label classification usually do not scale to such a large number of labels and adhoc methods are often employed. Furthermore, the marginal distribution of labels across all instances exhibits a long tail, which causes additional statistical challenges. 

Algorithmic approaches to XMC include optimized one-versus-all methods~\cite{babbar2017dismec,babbar2019,yen2017ppdsparse, yen2016pd}, embedding-based methods~\cite{bhatia2015,tagami2017annexml,guo2019}, probabilistic label tree-based~\cite{prabhu2018parabel,jasinska2016extreme,khandagale2019bonsai,wydmuch2018no} and deep learning-based methods~\cite{you2019attentionxml,liu2017deep,you2019haxmlnet,chang2019xbert}. Each algorithm usually proposes a specific approach to model the text as well as deal with tail labels. For example,~\citet{babbar2019} uses a robust SVM approach on TF-IDF features. PfastreXML~\cite{jain2016extreme} assumes a particular noise model for the observed labels and proposes to weight the importance of tail labels. AttentionXML~\cite{you2019attentionxml} uses a bidirectional-LSTM to embed the raw text as well as a multi-label attention mechanism to help capture the most relevant part of the input text for each label. For datasets with large $L$, AttentionXML trains one model per layer of a shallow and wide probabilistic latent tree using a small set of candidate labels. 

\subsection{Batch Learning from Bandit Feedback (BLBF)} We assume that the instance $x$ is sampled from a distribution $\Prob(x)$. The action for this particular instance is a unique label $y\in [L]$, sampled from the logging policy $\rho(y \mid x)$ and a feedback value $r \in \R$ is observed. Repeating this data collection process yields the dataset $\left[(x_i, y_i, r_i)\right]_{i=1}^n$. The BLBF problem consists in maximizing the expected reward $V(\pi)$ of a policy $\pi$. We use importance sampling (IS) to estimate $V(\pi)$ from data based on the logging policy as follows: 
\begin{align}
\label{eq:IS_one_action}
    \hat{V}_\text{IS}(\pi) = \frac{1}{n}\sum_{i=1}^n\frac{\pi(y_i \mid x_i)}{\rho(y_i \mid x_i)}r_i.
\end{align}
Classically, identifying the optimal policy via this estimator is infeasible without a thorough exploration of the action space by the logging policy~\cite{langford2008exploration}. More specifically, the IS estimator \smash{$\hat{V}_\text{IS}(\pi)$} requires the following basic assumption for there to be any hope of asymptotic optimality:
\begin{asp}
\label{ass:overlap} There exists a scalar $\epsilon > 0$ such that for all $x \in \R^d$ and $y \in [L], \rho(y \mid x) > \epsilon$.
\end{asp}
The IS estimator has high variance when $\pi$ assigns actions that are infrequent in $\rho$; hence a variety of regularization schemes have been developed, based on risk-upper-bound minimization, to control variance. Examples of upper bounds include empirical Bernstein concentration bounds~\cite{Swaminathan2015CF} and various divergence-based bounds~\cite{atan2018counterfactual,Wu2018,Johansson2016,lopez2019cost}. 
Another common strategy for reducing the variance is to propose a model of the reward function, using as a baseline a doubly robust estimator~\cite{Dudik2011, su2019cab}.

A recurrent issue with BLBF is that the policy may avoid actions in the training set when the rewards are not scaled properly; this is the phenomenon of \emph{propensity overfitting}. \citet{Swaminathan2015self} tackled this problem via the self-normalized importance sampling estimator (SNIS), in which IS estimates are normalized by the average importance weight. SNIS is invariant to translation of the rewards and may be used as a safeguard against propensity overfitting. BanditNet~\cite{Joachims2018} made this approach amenable to stochastic optimization by translating the reward distribution:
\begin{align}
\label{eq:BN_one_action}
    \hat{V}_\text{BN}(\pi) = \frac{1}{n}\sum_{i=1}^n\frac{\pi(y_i \mid x_i)}{\rho(y_i \mid x_i)}\left[r_i - \lambda\right],
\end{align}
and selecting $\lambda$ over a small grid based on the SNIS estimate of the policy value. 

Learning an XMC algorithm from bandit feedback requires offline learning from slates $\bm{Y}$, where each element of the slate comes from a large action space. \citet{swaminathan2017off} proposes a pseudo-inverse estimator for offline learning from combinatorial bandits. However, such an approach is intractable for large action spaces as it requires inverting a matrix whose size is linear in the number of actions. Another line of work focuses on offline evaluation and learning of semi-bandits for ranking~\cite{li2018offline,joachims2017unbiased} but only with a small number of actions. In real-world data, a partial matching strategy between instances and relevant actions is applied in applications to internet marketing for policy evaluation~\cite{wang2016beyond,li2015toward}. More recently,~\citet{chen2019topk} proposed a top-k off-policy correction method for a real-world recommender system. Their approach deals with millions of actions although it treats label embeddings as given, whereas this problem is in general a hard problem for XMC. 

\section{Bandit Feedback and Multi-label Classification}
\label{sec:bandit-mulilabel}
We consider a setting in which the algorithm (e.g., a policy for a recommendation system) observes side information $x \in \mathbb{R}^d$ and is allowed to output a subset $\bm{Y} \subseteq [L]$ of the $L$ possible labels. Side information is independent at each round and sampled from a distribution $\Prob(x)$. We assume that the subset $\bm{Y}$ has fixed size $|\bm{Y}| = \ell$, which allows us to adopt the slate notation $\bm{Y} = (y_1, \ldots, y_\ell)$. The algorithm observes noisy feedback for each label, $\bm{R} = (r_1, \ldots, r_\ell)$, and we further assume that the joint distribution over $\bm{R}$ decomposes as 
\smash{\(
    \Prob\left(\bm{R} \mid x, \bm{Y}\right) = \prod_{j=1}^\ell  \Prob\left(r_j \mid x, y_j\right).
\)}
We will denote the conditional reward distribution as a function: \(\delta(x, y) = \E[r \mid x, y].\) In the case of multi-label classification, this feedback can be formed with random variables indicating whether each individual label is inside the true set of labels for each datapoint~\cite{gentile2014multilabel}. More concretely, feedback may be formed from sale or click information~\cite{chen2019topk,chen2019gan}. 

We are interested in optimizing a policy $\pi(\bm{Y} \mid x)$ from offline data. Accessible data is sampled according to an existing algorithm, the logging policy $\rho(\bm{Y} \mid x)$. We assume that both joint distributions over the slate decompose into an auto-regressive process. For example, for $\pi$ we assume:
\begin{align}
\label{eq:pi_decomposition}
    \pi(\bm{Y} \mid x) &=\prod_{j=1}^\ell \pi(y_j \mid x, y_{1:j-1}).
\end{align}
Introducing this decomposition does not result in any loss of generality, as long as the action order is identifiable (otherwise, one would need to consider all possible orderings~\cite{kool2020estimating}). This is a reasonable hypothesis because the order of the actions may also be logged as supplementary information. We now define the value of a policy $\pi$ as:
\begin{align}
\label{eq:value_pi}
    V(\pi) &= \mathbb{E}_{\Prob(x)}\mathbb{E}_{\pi(\bm{Y} \mid x)}\mathbb{E}\left[\bm{1}^\top \bm{R} \mid x, \bm{Y}\right].
\end{align}
In our setting, the reward decomposes as a sum of independent contributions of each individual action. The reward may in principle be generalized to be rank dependent, or to consider interactions between items~\cite{gentile2014multilabel}, but this is beyond the scope of this work.

A general approach for offline policy learning is to estimate $V(\pi)$ from logged data using importance sampling~\cite{Swaminathan2015CF}. As emphasized in~\citet{swaminathan2017off}, the combinatorial size of the action space $\Omega(L^\ell)$ may yield an impractical variance for importance sampling. This is particularly the case for XMC, where typical values of $L$ are minimally in the thousands. A natural strategy to improve over the IS estimator on the slate $\bm{Y}$ is to exploit the additive reward decomposition in Eq.~\eqref{eq:value_pi}. Along with the factorization of the policy in Eq.~\eqref{eq:pi_decomposition}, we may reformulate the policy value as:
\begin{align}
\label{eq:decom_value_pi}
    V(\pi) &= \mathbb{E}_{\Prob(x)}\sum_{j=1}^\ell\mathbb{E}_{\pi(y_{1:j} \mid x)}\delta(x, y_j).
\end{align} 
The benefit of this new decomposition is that instead of performing importance sampling on $\bm{Y}$, we can now use $\ell$ IS estimators, each with a better bias-variance tradeoff. Unbiased estimation of $V(\pi)$ in Eq.~\eqref{eq:decom_value_pi} via importance sampling still requires Assumption~\ref{ass:overlap}. The logging policy must therefore explore a large action space. However, most actions are unrelated to a given context and deploying an online logging policy that satisfies Assumption~\ref{ass:overlap} may yield a poor customer experience.

\section{Learning from eXtreme Bandit Feedback}
\label{sec:learning_XBF}

We now explore alternative assumptions for the logging policy that may be more suitable to the setting of very large action spaces. We formalize the notion that most actions are irrelevant using the following assumption:
\begin{asp}\emph{(Sparse feedback condition).} \label{ass:sparse_feedback} The individual feedback random variable $r$ takes values in the bounded interval $[\lowr, \highr]$. For all $x \in \R^d$, the label set $[L]$ can be partitioned as $[L] =\Psi(x) \coprod \Psi^0(x)$ such that for all actions $y$ of $\Psi^0(x)$, the expected reward is minimal: $\delta(x, y) = \lowr$. 
\end{asp}
We refer to the function $\Psi$ as an \emph{action selector}, as it maps a context to a set of relevant actions. Throughout the manuscript, we use the notation $\Lambda^0$ to refer to the pointwise set complement of any action selector $\Lambda$. Intuitively, we are interested in the case where $\abs*{\Psi(x)} \ll L$ for all $x$. Assumption~\ref{ass:sparse_feedback} is implicitly used in online marketing applications of offline policy evaluation, formulated as a partial matching between actions and instances~\cite{wang2016beyond,li2015toward}. Notably, this assumption can be assimilated to a mixed-bandit feedback setting, where we observe feedback for all of $\Psi^0(x)$ but only one selected action inside of $\Psi(x)$. Under Assumption~\ref{ass:sparse_feedback}, the IS estimator will be unbiased for all logging policies that satisfy the following relaxed assumption:
\begin{asp}\emph{($\Psi$-overlap condition).} \label{ass:extended_overlap} There exists a scalar $\epsilon > 0$ such that for all $x \in \R^d$ and $y \in \Psi(x)$, $\rho(y \mid x) > \epsilon$.
\end{asp}
Batch learning from bandit feedback may be possible under this assumption, as long as the logging policy explores a set of actions large enough to cover the actions from $\Psi$ but small enough to avoid exploring too many suboptimal actions. Furthermore, Assumption~\ref{ass:sparse_feedback} also reveals the existence of $\Psi^0(x)$, a  sufficient statistic for estimating the reward on the irrelevant actions. Making appeal to Rao-Blackwellization~\cite{casella1996rao}, we can incorporate this information to estimate each of the $\ell$ terms of Eq.~\eqref{eq:decom_value_pi} (e.g., in the case $\ell=1$ and $\nabla=0$): 
\begin{align}
     V(\pi) &= \mathbb{E}_{\Prob(x)}\left[\pi\left(\Psi(x) \mid x\right)\cdot\mathbb{E}_{\pi(y \mid x)} \left[\delta(x, y) \mid y \in \Psi(x) \right]\right].
    \label{eq:value_decomposed}
\end{align}
The decomposition in Eq.~\eqref{eq:value_decomposed} suggests that when the action selector $\Psi$ is known, one can estimate $V(\pi)$ via importance sampling for the conditional expectation of the rewards with respect to the event $\{y \in \Psi(x)\}$. Intuitively, this means that one can modify the importance sampling scheme to only include a relevant subset of labels and ignore all the others. Without loss of generality, we assume that $\lowr=0$ in the remainder of this manuscript. 

In practice, the oracle action selector $\Psi$ is unknown and needs to be estimated from the data. It may be hard to infer the smallest $\Psi$ such that Assumption~\ref{ass:sparse_feedback} is satisfied. Conversely, a trivial action selector including all actions is valid (it does include all relevant actions) but is ultimately unpractical. As a flexible compromise, we will replace $\Psi$ in Eq.~\eqref{eq:value_decomposed} by any action selector $\Phi$ and study the bias-variance tradeoff of the resulting plugin estimator.

Let $\rho$ be a logging policy with a large enough support to satisfy Assumption~\ref{ass:extended_overlap}. Let $\Phi$ be an action selector such that $\Phi(x) \subset \textrm{supp}~\rho(\cdot \mid x)$ almost surely in $x$, where $\textrm{supp}$ denotes the support of a probability distribution. The role of $\Phi$ is to prune out actions to maintain an optimal bias-variance tradeoff. In the case $\ell=1$, the $\Phi$-selective importance sampling (sIS) estimator \smash{$\VsIPS{\pi}$} for action selection $\Phi$ can be written as:
\begin{align}
\VsIPS{\pi} = \frac{1}{n}\sum_{i=1}^n\frac{\pi(y_i \mid x_i, y \in \Phi(x))}{\rho(y_i \mid x_i)}r_i.
\end{align}
Its bias and variance depends on how different the policy $\pi$ is from the logging policy $\rho$ (as in classical BLBF) but also on the degree of overlap of $\Phi$ with $\Psi$:
\begin{restatable}[Bias-variance tradeoff of selective importance sampling]{theorem}{biasvar} \label{thm:bias-var} Let $\bm{R}$ and $\rho$ satisfy Assumptions~\ref{ass:sparse_feedback} and~\ref{ass:extended_overlap}. Let $\Phi$ be an action selector such that $\Phi(x) \subset \textrm{supp}~\rho(\cdot \mid x)$ almost surely in $x$. The bias of the sIS estimator is:
\begin{align}
    \abs*{\E\VsIPS{\pi} - V(\pi)} & \leq \highr \kappa(\pi, \Psi, \Phi),
\end{align}
where \smash{$\kappa(\pi, \Psi, \Phi) = \E_{\Prob(x)} \pi\left(\Psi(x) \cap \Phi^0(x) \mid x\right)$} quantifies the overlap between the oracle action selector $\Psi$ and the proposed action selector $\Phi$, weighted by the policy $\pi$. 
The performance of the two estimators can be compared as follows:
\begin{align}
    \MSE{\VsIPS{\pi}} \leq&~\MSE{\VIPS{\pi}} + 2\highr^2\kappa(\pi, \Psi, \Phi) \\ &- \frac{\sigma^2}{n}\E_{\Prob(x)}\frac{\pi^2\left(\Phi^0(x) \mid x\right)}{\rho\left(\Phi^0(x) \mid x\right)},
    \label{eq:lower-bound}
\end{align}
where $\sigma^2 = \inf_{x, y \in \R^d \times [K]}\E[r^2 \mid x, y]$.
\end{restatable}

We provide the complete proof of this theorem in the appendix. As expected by Rao-Blackellization, we see that if $\Phi$ completely covers $\Psi$ (i.e., for all $x \in \R^d, \Psi(x) \subset \Phi(x)$), then \smash{$\VsIPS{\pi}$} is unbiased and has more favorable performance than \smash{$\VIPS{\pi}$}. Admittedly, Eq.~\eqref{eq:lower-bound} shows that both estimators have similar mean-square error when $\pi$ puts no mass on potentially irrelevant actions $y \in \Phi^0(x)$. However, during the process of learning the optimal policy or in the event of propensity overfitting, we expect $\pi$ to put a non-zero mass on potentially irrelevant actions $y \in \Phi^0(x)$, with positive probability in $x$. For these reasons, we expect \smash{$\VsIPS{\pi}$} to provide significant improvement over \smash{$\VIPS{\pi}$} for policy learning. 

Even though Eq.~\eqref{eq:lower-bound} provides insight into the performance of sIS, unfortunately it cannot be used directly in selecting $\Phi$. We instead propose a greedy heuristic to select a small number of action selectors. For example, $\Phi^p(x)$ corresponds to the top-$p$ labels for instance $x$ according to the logging policy. With this approach, the bias of the \smash{$\VsIPS{\pi}$} estimator is a decreasing function of $p$, as the overlap with $\Psi$ increases. Furthermore, the variance increases with $p$ as long as the added actions are irrelevant. In practice, we use a small grid search for $p \in \{10, 20, 50, 100\}$ and choose the optimal $p$ with the SNIS estimator, as in BanditNet. We believe this is a reasonable approach whenever the logging policy ranks the relevant items sufficiently high but can be improved (e.g., top-$p$ for $p$ in 5 to 100).

\begin{figure*}[ht]
\begin{subfigure}{.3\textwidth}
  \centering
  \includegraphics[width=\linewidth]{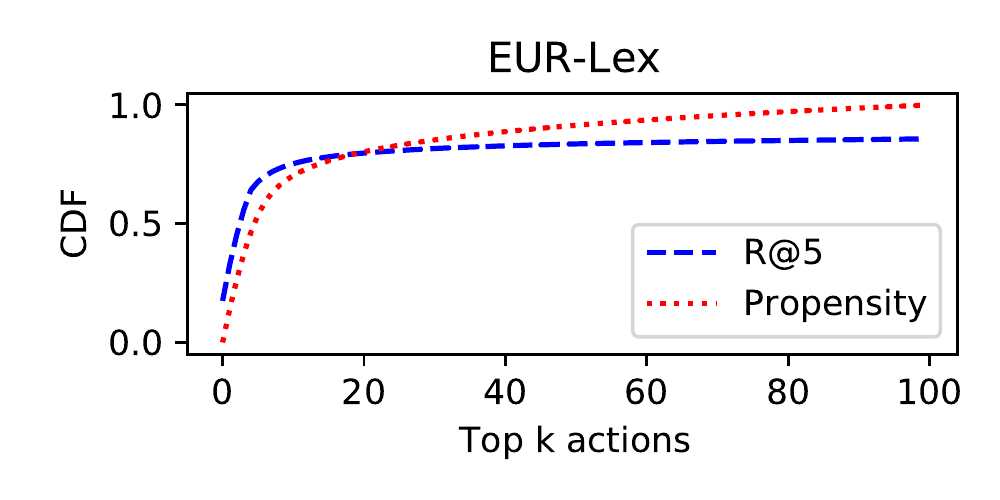}
\end{subfigure}%
\begin{subfigure}{.3\textwidth}
  \centering
  \includegraphics[width=\linewidth]{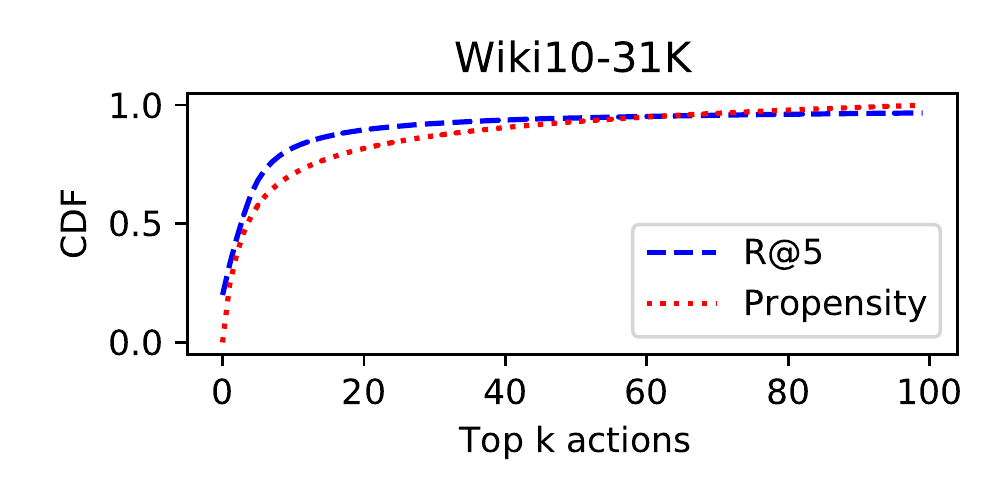}
\end{subfigure}%
\begin{subfigure}{.3\textwidth}
  \centering
  \includegraphics[width=\linewidth]{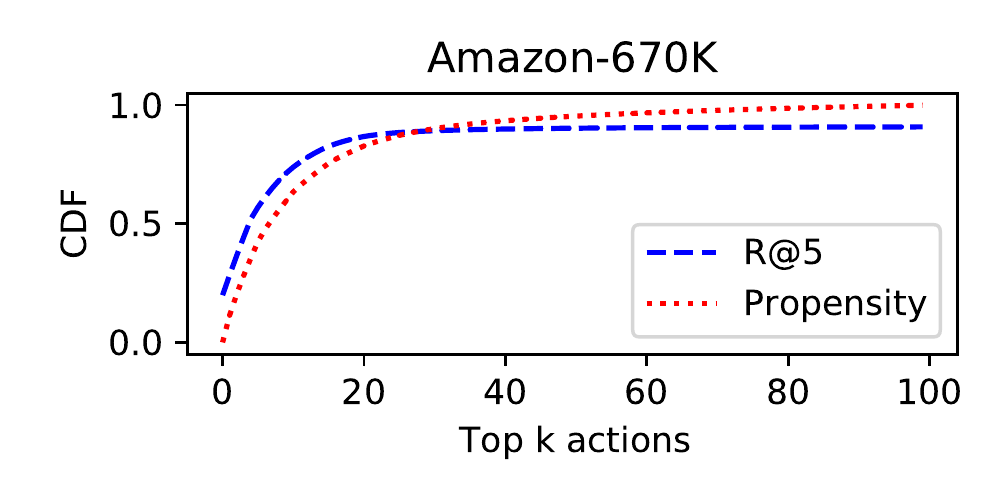}
\end{subfigure}%
    \caption{Expected $R$@5 and CDF of the logging policy for the top-$k$ action for each XMC dataset. Exploration is limited to a subset of relevant actions.}
    \label{fig:covering}
% \vspace{-0.5cm}
\end{figure*}

\section{Policy Optimization for eXtreme Models}
\label{sec:POXM}

We apply sIS for each of the $\ell$ terms of the policy value from Eq.~\eqref{eq:decom_value_pi} in order to estimate $V(\pi)$ from bandit feedback, $(x_i, \bm{Y}_i, \bm{R}_i)_{i=1}^n$, and learn an optimal policy. As an additional step to reduce the variance, we prune the importance sampling weights of earlier slate actions, following~\citet{Achiam2017}:
\begin{align}
\label{eq:obj_not_corrected}
    \hat{V}_\textrm{sIS}^\Phi(\pi) = \frac{1}{n}\sum_{i=1}^n\sum_{j=1}^\ell\frac{\pi^\Phi(y_{i, j} \mid x_i, y_{1:j-1})}{\rho(y_{i, j} \mid x_i, y_{1:j-1})}r_{i, j},
\end{align}
where $\pi^\Phi$ designates the distribution $\pi$ restricted to the set $\Phi(x)$ for every $x$. Because $\pi(\bm{Y} \mid x)$ is a copula, one can derive all joint distributions starting from
the corresponding one-dimensional marginals~\cite{sklar1959fonctions}. In this work, we focus on the case of ordered sampling without replacement to respect an important design restriction: the slate $\bm{Y}$ must not have redundant actions. For the $j$-th slate component, the relevant conditional probability is formed from the base marginal probabilities $\pi(y \mid x)$ as follows:
\begin{align}
\label{eq:no-replacement}
    \pi^\Phi(y_{j} \mid x, y_{1:j-1}) %= \pi(y_{j} \mid x, y_{1:j-1}, y_{1:j} \subset \Phi(x) ) 
    = \frac{\pi(y_j \mid x)}{\displaystyle\sum_{y' \in \Phi(x)}\pi(y' \mid x) - \sum_{k < j}\pi(y_k \mid x)}.
\end{align}
From a computational perspective, the action selector also diminishes the computational burden, leading to efficient computations of the probabilities when the marginals are parameterized by a softmax distribution. Indeed, Eq.~\eqref{eq:no-replacement} depends only on the logits for the actions inside of the set $\Phi$. This helps our approach to scale to large XMC datasets. 

As mentioned in the background section, directly maximizing the importance sampling estimate of the policy value in Eq.~\eqref{eq:obj_not_corrected} may be pathological due to propensity overfitting. The BanditNet approach may be adapted to the slate case using a different loss translation scheme for each element:
\begin{align}
\label{eq:obj_corrected}
    \hat{V}_\textrm{sIS}^\Phi(\pi) = \frac{1}{n}\sum_{i=1}^n\sum_{j=1}^\ell\frac{\pi^\Phi(y_{i, j} \mid x_i, y_{1:j-1})}{\rho(y_{i, j} \mid x_i, y_{1:j-1})}[r_{i, j} - \lambda_j],
\end{align}
and with $(\lambda_1, \ldots, \lambda_\ell)$ selected out of a small grid based on the self-normalized importance sampling estimate of the policy value from the training data~\cite{Joachims2018}. For computational reasons, we only search for a unique $\lambda$ and, following~\citet{Joachims2018}, we focus on the grid $\{0.7, 0.8, 0.9, 1.0\}$. We refer to this approach as \textit{Policy Optimization for eXtreme Models} (POXM), named after the seminal algorithm from~\citet{Swaminathan2015CF}.

\section{Experiments}
\label{sec:experiments}
We evaluate our approach on real-world datasets with a supervised learning to bandit feedback conversion~\cite{Dudik2011, gentile2014multilabel}. We report results on three datasets from the Extreme Classification Repository~\cite{Bhatia16}, with $L$ ranging from several thousand to half a million (Table \ref{tab:dataset}). EUR-Lex~\cite{mencia2008efficient} has a relatively small label set and each instance has a sparse label set. Wiki10-31K~\cite{zubiaga2012enhancing} has a larger label set as well as more abundant annotations. Finally, Amazon-670K~\cite{mcauley2013hidden} has more than half a million labels. To our knowledge, this is the first time that such action spaces have been considered for BLBF.

\begin{table}[t]
 \caption{XMC datasets used for semi-simulation of eXtreme bandit feedback.}
% \vspace{0.1in}
\begin{center}
\begin{small}
% \begin{sc}
\label{tab:dataset}
\scalebox{0.8}{
\begin{tabular}{@{}crrrrrrrrrrrrrrrr@{}}
\toprule
Dataset & $N_\text{train}$ & $N_\text{test}$ & $D$ & $L$ & $\overline{L}$ & $\hat{L}$  \\
\midrule
% EUR-Lex~\cite{mencia2008efficient}
EUR-Lex & 15,449 & 3,865 & 186,104 & 3,956 & 5.30 & 20.79 \\
% Wiki10-31K~\cite{zubiaga2012enhancing}
Wiki10-31K & 14,146 & 6,616 & 101,938 & 30,938 & 18.64 & 8.52 \\
% AmazonCat-13K~\cite{mcauley2013hidden} & 1,186,239 & 306,782 & 203,882 & 13,330 & 5.04 & 448.57 \\
% Amazon-670K~\cite{mcauley2013hidden}
Amazon-670K & 490,449 & 153,025 & 135,909 & 670,091 & 5.45 & 3.99 \\
% Wiki-500K~\cite{Bhatia16} & 1,779,881 & 769,421 & 2,381,304 & 501,008 & 4.75 & 16.86 \\
\bottomrule
\end{tabular}
}
% \end{sc}
\end{small}
\end{center}
 $N_\text{train}$: \#training instances,
 $N_\text{test}$: \#test instances, 
 $D$: \#features, 
 $L$: \#labels and size of the action space,
 $\overline{L}$: average \#labels per instance,
 $\hat{L}$: the average \#instances per label. The partition of training and test is from the data source.
% \vspace{-0.3cm}
\end{table}
\begin{table*}[ht]
\caption{Performance comparisons of POXM and other competing methods over the three medium-scale datasets. All experiments are conducted with bandit feedback. In italic are the results from the AttentionXML manuscript, for the full-information feedback on all the training data (the supervised learning skyline).}
% \vspace{0.1in}
\begin{center}
\begin{small}
% \begin{sc}
\label{tab:medium-scale}
	\centering
\scalebox{1.}{
	\begin{tabular}{@{}lcccccccccccccccccccccccccccccc@{}}
		\toprule
		Methods & R@3 & R@5&nDCR@3&nDCR@5&PSR@3 & PSR@5 \\
		\midrule \midrule
		& \multicolumn{6}{c}{EUR-Lex}\\
		\hline 
		Logging policy         & 33.79 & 31.23 & 33.77 & 34.07 &  22.33 & 21.66 \\
		Direct Method          & 39.58 &  32.22 & 42.64 & 38.69 & 25.81 & 26.58 \\
		BanditNet      & 15.60 & 13.51 & 17.68 & 16.29 & 8.58 & 8.48 \\
		PM-BanditNet   & 20.44 & 15.13 & 24.17 & 20.42 &9.51 & 9.52 \\
		POXM      & \textbf{52.38} & \textbf{44.48} & \textbf{55.73} & \textbf{51.64} & \textbf{35.42} & \textbf{35.25}  \\
		\textit{AttentionXML} & \textit{73.08} & \textit{61.10} & \textit{76.37} & \textit{70.49} & \textit{51.29} & \textit{53.86}\\
		\midrule \midrule
		 & \multicolumn{6}{c}{Wiki10-31K} \\
		 \hline  
		Logging policy         & 42.49 & 38.80 & 43.57 & 41.13 & 7.43 & 7.46 \\
		Direct Method          & 48.96 & 38.16 & 55.72 & 46.38 & 8.22 & 7.78 \\
		BanditNet      & 49.92 & 36.16 & 56.54 & 45.08 & 7.20 & 7.20  \\
		PM-BanditNet  & 49.06 & 37.04 & 55.91 & 45.74 & 7.09 & 7.04\\
		POXM & \textbf{60.45} & \textbf{53.03} & \textbf{64.22} & \textbf{58.26} & \textbf{10.70} & \textbf{10.58} \\
		\textit{AttentionXML} & \textit{77.78} & \textit{68.78} & \textit{79.94} & \textit{73.19} & \textit{17.05} & \textit{17.93} \\
		 \midrule \midrule
		 & \multicolumn{6}{c}{Amazon-670K}\\
		 		\hline 
		Logging policy         & 17.89 & 17.05 & 18.77 & 18.65 & 13.06 & 13.06 \\
		Direct Method          &
		23.42 & 20.14 & 25.16 & 23.28 & 15.82 & 16.30 \\
		BanditNet   & 16.83 & 14.54 & 17.18 & 16.11 & 11.67 & 11.67  \\
		PM-BanditNet & 17.31 & 14.76 & 17.67 & 16.42 & 12.05 & 12.05\\
		POXM      & \textbf{26.89} & \textbf{23.72} & \textbf{28.93} & \textbf{27.22} & \textbf{19.59} & \textbf{20.75} \\
		\textit{AttentionXML} & \textit{40.67} & \textit{36.94} & \textit{43.04} & \textit{41.35} & \textit{32.36} & \textit{35.12}\\
		\bottomrule
		\end{tabular}
}
% \end{sc}
\end{small}
\end{center}
% \vspace{-0.3cm}
\end{table*}

\subsection{Simulating Bandit Feedback from XMC datasets}
An XMC dataset is a collection of observations $(x_i, \bm{Y_i}^*)_{i=1}^n$ for which each instance $x_i$ is associated with an optimal set of labels $\bm{Y_i}^*$. To form a logging policy $\rho$, we train AttentionXML on a small fraction $\alpha$ of the dataset to get estimates of the marginal probability for each label (values are provided in the appendix). These probabilities must be normalized in order to sum to one, as expected in the multi-label setting~\cite{wydmuch2018no}. The ground-truth labels may be used to investigate whether $\Phi^p$ (the top $p$ actions from $\rho$) approximately satisfies the $\Psi$-covering condition. On the EUR-LeX dataset the obtained logging policy on its top $20$ action covers around 75\% of the rewards (Figure~\ref{fig:covering}). Using more actions may be suboptimal as these may add variance with only a marginal benefit on the bias, as captured by Theorem~\ref{thm:bias-var}. Finally, we form bandit feedback for slates of size $\ell$ by sampling without replacement from $\rho$. The reward is a binary variable depending on whether the chosen action belongs to the reference set $\bm{Y}^*$. We fix $\ell=5$ in all experiments.

\begin{figure*}[ht]
\begin{subfigure}{.3\textwidth}
  \centering
  \includegraphics[width=\linewidth]{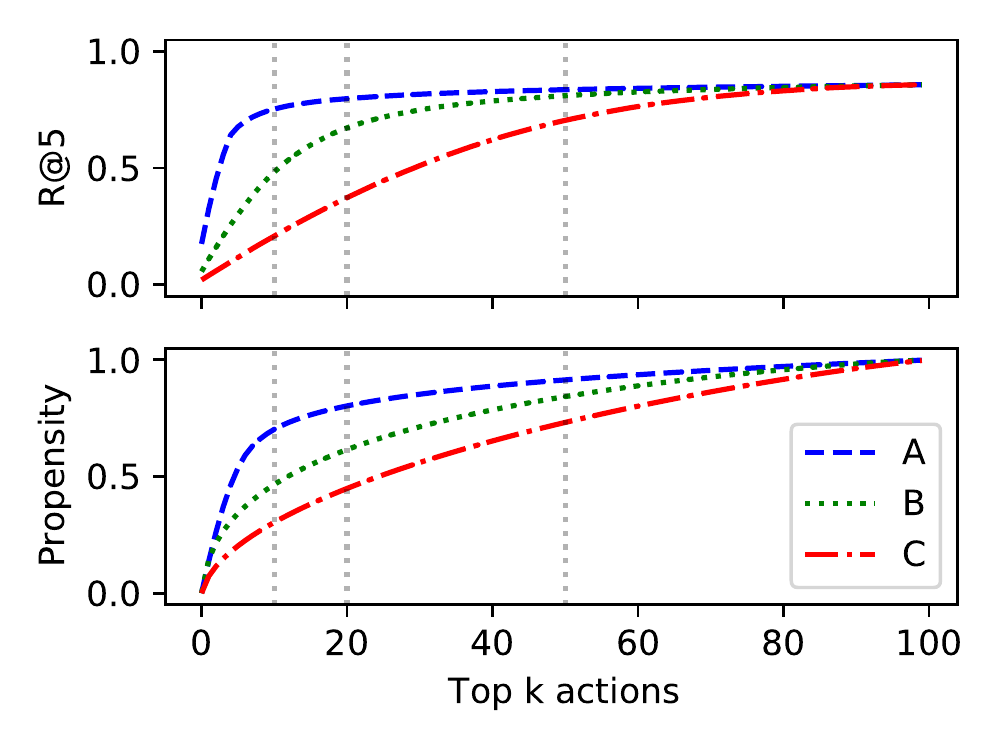}
\end{subfigure}
\hfill
\begin{subfigure}{.3\textwidth}
  \centering
  \includegraphics[width=\linewidth]{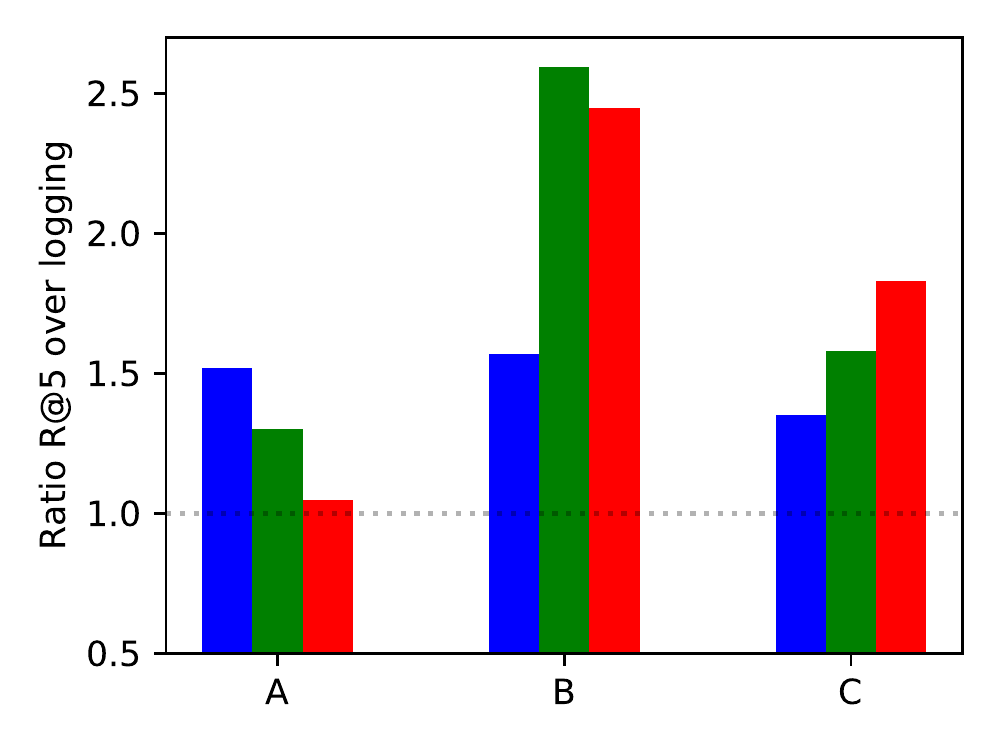}
\end{subfigure}
\hfill
\begin{subfigure}{.3\textwidth}
  \centering
  \includegraphics[width=\linewidth]{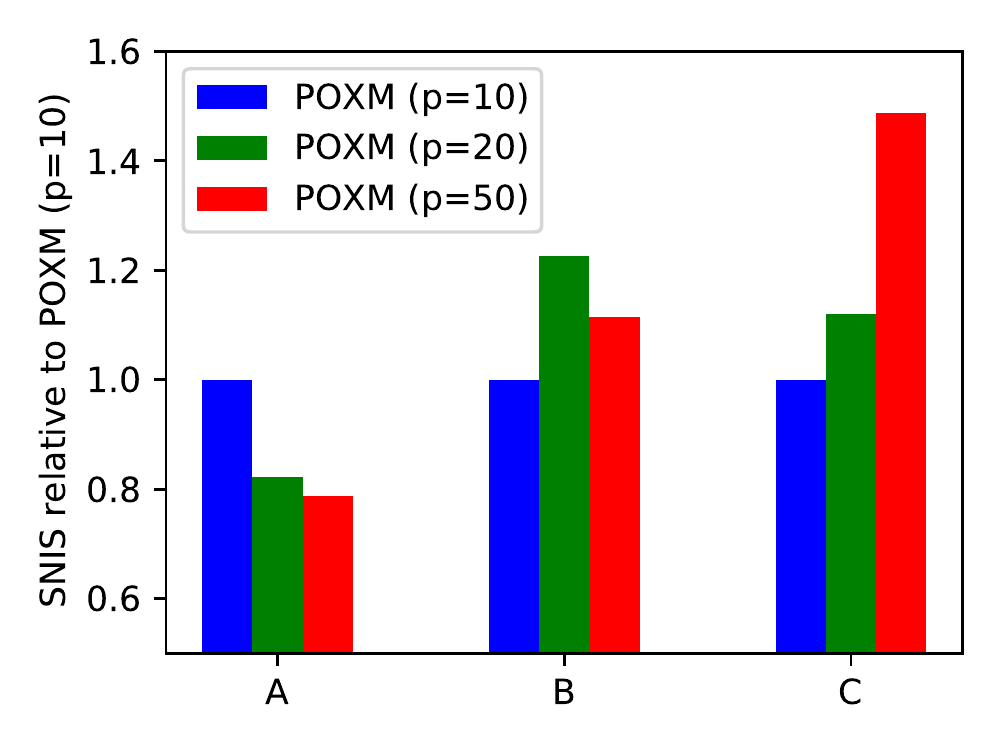}
  \end{subfigure}%
    \caption{Data-driven selection of $p$ on the EUR-LeX dataset. Left: logging policy statistics for three randomization scenarios (A, B, C, described in appendix). Middle: R@5 performance for each POXM variant and each logging policy. Right: SNIS estimates used for selection of $p$ in POXM. }
    \label{fig:robustness}
    % \vspace{-0.5cm}
\end{figure*}

\subsection{Evaluation metrics}

P@$k$ (Precision at $k$), nDCG@$k$ (normalized Discounted Cumulative Gain at $k$) as well as PSP@$k$ (Propensity Scored Precision at $k$) are widely used metrics for evaluating XMC methods~\cite{jain2016extreme,Bhatia16}. We adapt these metrics to the evaluation of stochastic policies by taking expectations of the relevant statistics over slates of size $k$ (with distinct items). For example, R@$k$ (Reward at $k$) is defined as:
\begin{align}
\text{R@}k = \E_{\pi(y_1, \ldots, y_k)} \frac{1}{k}\sum_{l=1}^k \mathds{1}\{y_l \in \bm{Y}^*\}.
\end{align}
Similarly, we define nDCR@$k$ and PSR@$k$ (analogous to nDCG@$k$ and PSP@$k$). We estimate those metrics using sampling without replacement.

\subsection{Competing methods and experimental settings}
We compare POXM to other offline policy learning methods in the specific context of AttentionXML~\cite{you2019attentionxml}. Furthermore, hyperparameters that are specific to AttentionXML are fixed across all experiments (Table~2 of~\citet{you2019attentionxml}) so that our results are not confounded by those choices. To reduce training time and focus on how well each method deals with large action spaces, we use the LSTM weights from AttentionXML and treat those as fixed for all experiments. Finally, we noticed that the scale of gradients of the objective function for IS-based methods was different from supervised learning methods (similarly reported in~\citet{Joachims2018}). Consequently, we lowered the learning rate for these algorithms from $1\mathrm{e}{-4}$ to $5\mathrm{e}{-5}$. 

We compare POXM to several baselines. First, we report results for the Direct Method (DM), a supervised learning baseline where AttentionXML is trained with a partial classification loss, using only the feedback from $\ell$ actions for each instance. The deterministic policy picks the top-$k$ actions from the predicted value, akin to~\citet{prabhu2020extreme}. Second, we use BanditNet as a baseline. For this, we train AttentionXML using gradients of Eq.~\eqref{eq:obj_not_corrected}, but without conditioning on action set $\Phi$. Instead, we use the full softmax (akin to~\citet{Joachims2018}) or we approximate it with negative sampling~\cite{mikolov2013distributed} at training time (only for the Amazon-670K dataset). Finally, we also investigate the effect of the partial matching strategy of~\citet{wang2016beyond,li2015toward} while training with BanditNet (referred to as BanditNet-PM). In this baseline we ignore feedback from actions that are not in $\Phi^p$. 

\subsection{Results}
Table \ref{tab:medium-scale} shows the performance results of POXM and other competing
methods. POXM consistently outperformed the logging policy and always significantly improved over the competing methods. As expected, the performance is lower than AttentionXML learned on the full training set. The direct method also improved over the logging policy but only marginally, which is attributable to the bias from the logging policy. BanditNet and its partial matching variant did not improve over the logging policy on both EUR-Lex and Amazon-670K. We believe this is due to the sparsity of the rewards. Indeed, BanditNet outperforms the logging policy as well as the Direct Method baseline on Wiki10-31K that has many more labels per instance. Furthermore, we see that partial matching has a positive effect on BanditNet for EUR-LeX but not for the other datasets. 

For all choices of logging policy in Figure~\ref{fig:covering}, the optimal value of $p$ selected by POXM is the smallest possible ($p =10$). Therefore, we investigated how the algorithm behaved with more stochastic policies on the EUR-LeX dataset. For this, we injected Gumbel noise into the label probabilities (details in the appendix) and analyzed the performance of POXM for logging policies with $p \in \{10, 20, 50\}$. We provide summary statistics for the three logging policies and report the results of POXM in Figure~\ref{fig:robustness}. We see that each logging policy has a best performing value of $p$ (middle) that is aligned with the summary statistics of the logging policy (left) as well as the normalized importance sampling (SNIS) policy value estimate (right). This shows that POXM keeps improving over the logging policy for more stochastic policies and that SNIS is a reasonable procedure for selecting the parameter $p$.

\section{Discussion}
We have presented POXM: a scalable algorithmic framework for learning XMC classifiers from bandit feedback. On real-world datasets, we have shown that POXM is systematically able to improve over the logging policy.  This is not the case for the current state-of-the-art method, BanditNet.  The latter does not always improve over the logging policy, which may be attributable to propensity overfitting.

All public datasets for eXtreme multi-label classification present the problem of imbalanced label distribution. Indeed, certain important labels (commonly referred to as \textit{tail labels}), with more descriptive power, might be rarely used because of biases inherent to the data collection process. Although we do not provide a specific treatment of tail labels in this manuscript, we proposed in the appendix a simple extension of POXM (named wPOXM) based on~\citet{jain2016extreme} to address this problem. Briefly, we extended the traditional data generating process for BLBF to treat the labels as noisy, and assumed that our observation scheme is biased towards the head labels. This leads to a slight modification of the sIS estimator and the POXM procedure to include the label propensity scores. wPOXM significantly improved over POXM for all propensity-weighted metrics, with 4.77\% improvement of the PSR@3 metric. We leave more refined analyses for future work.

An important point in the XMC literature is computational efficiency. In this study, we used a machine with 8 GPUs Tesla K80 to run our experiments. This is mainly because our implementation relies on AttentionXML, itself implemented in PyTorch. The runtime of 
POXM on each dataset ranges from less than one hour for EUR-LeX to less than three hours for Amazon-670K. An important aspect of POXM's implementation is the reduced softmax computation. We verified this on the Amazon-670K dataset in which we tracked the runtime for growing size of the parameter $p$. For less than $p \leq 100$ actions, POXM took around 3s to backpropagate through 1,000 samples. However, this runtime was multiplied by ten for $p$=10,000 (25s) and we could not run POXM for $p \geq$ 20,000 because of an out-of-memory error. An interesting research direction would be to apply this framework to other XMC algorithms.

A performance gap remains between POXM and the skyline performance from the supervised method AttentionXML. 
It is possible that alternative parameterizations of the policy $\pi$ may further improve performance; for example, using a probabilistic latent tree for policy gradients as in~\citet{chen2019large} or using the Gumbel-Top-$k$ trick~\cite{kool2020ancestral}. Furthermore, doubly robust estimators~\cite{Dudik2011,su2019cab,wangbatch} may further help in incorporating prior knowledge about the reward function.

\section*{Acknowledgements}
We acknowledge Kush Batia, Lerrel Pinto, Adam Stooke, Sujay Sanghavi, Hsiang-Fu Yu, Arya Mazumdar and Rajat Sen for helpful conversations. 

\bibliography{references}

\newpage

\appendix
\onecolumn
\renewcommand{\thesection}{\Roman{section}}

\section{Proofs}
\label{sec:proofs}
\biasvar*

\begin{proof} 
We first derive an expression for the bias of the sIPS estimator under Assumption~\ref{ass:extended_overlap} for $\rho$ and Assumption~\ref{ass:sparse_feedback} for $\delta$:
\begin{align}
    \E\VsIPS{\pi} - V(\pi) &= \E_{\Prob(x)} \left[\E_{\rho(y \mid x)} \E_{\Prob(r \mid x, y)} \frac{\indyphi \pi(y \mid x)}{\rho(y \mid x)} r - \E_{\pi(y \mid x)}\delta(x, y)\right] \\
    &= \E_{\Prob(x)} \left[\E_{\rho(y \mid x)}\frac{\indyphi \pi(y \mid x)}{\rho(y \mid x)} \delta(x, y) - \E_{\pi(y \mid x)}\delta(x, y)\right] \\
    & = \E_{\Prob(x)\pi(y \mid x)} \left[\indynotphi \delta(x, y)\right] \\
    & = \E_{\Prob(x)\pi(y \mid x)} \left[\indypsi \indynotphi \delta(x, y)\right],
\end{align}
where we exploited the fact that importance sampling is unbiased for each $x$ on $\Phi(x)$ and that $\delta(x, y)$ is zero for $y \in \Psi^0(x)$. By taking absolute value on both sides, we can bound the bias as follows:
\begin{align}
\label{eq:bias_bound}
    \abs*{\E\VsIPS{\pi} - V(\pi)} & \leq \highr \E_{\Prob(x)} \pi\left(\Psi(x) \cap \Phi^0(x) \mid x\right)
\end{align}
We now relate the mean-square error of the IPS and sIPS estimators:
\begin{align}
    \MSE{\VIPS{\pi}} &= \MSE{\VsIPS{\pi}} + \E\left[ \VIPS{\pi}^2 - \VsIPS{\pi}^2 -2V(\pi)\left(\VIPS{\pi} - \VsIPS{\pi}\right) \right] \notag \\
    &= \MSE{\VsIPS{\pi}} + \underbrace{\E \VIPS{\pi}^2 - \E\VsIPS{\pi}^2}_{\textrm{Second-order moment difference}} + 2V(\pi) \underbrace{\left[\E\VsIPS{\pi} - V(\pi)\right]}_{\textrm{Negative bias}}.
\end{align}
Let us note $v(x, y) = \E[r \mid x, y]$ and $\sigma^2 = \inf_{x, y \in \R^d \times [K]}v(x, y)$. Let us focus for now on the second order moment difference in Equation~\eqref{eq:bias_bound}, which we can decompose as:
\begin{align}
    \E \VIPS{\pi}^2 - \E\VsIPS{\pi}^2 &=\frac{1}{n}\E_{\Prob(x)\rho(y \mid x)\Prob(r \mid x, y)} \frac{\indynotphi r^2 \pi^2(y \mid x)}{\rho^2(y \mid x)} \\
    &= \frac{1}{n}\E_{\Prob(x)\rho(y \mid x)} \frac{\indynotphi v(x, y) \pi^2(y \mid x)}{\rho^2(y \mid x)} \\
    &\geq \frac{\sigma^2}{n} \E_{\Prob(x)} \E_{\rho(y \mid x)} \frac{\indynotphi \pi^2(y \mid x)}{\rho^2(y \mid x)} \\
    & \geq \frac{\sigma^2}{n} \E_{\Prob(x)}\left[\frac{\pi^2\left(\Phi^0(x) \mid x\right)}{\rho\left(\Phi^0(x) \mid x\right)} \chidiv{\bar{\pi}(y \mid x)}{\bar{\rho}(y \mid x)}\right], \label{eq:lower-bound-chi}
\end{align}
where $\Delta_{\chi^2}$ denotes the chi-square divergence and $\bar{\pi}$ (resp. $\bar{\rho}$) denote the normalized probability distribution $\pi$ (resp. $\rho$) on the set $\Phi^0(x)$ for each $x$. The form of chi-square divergence in Equation~\eqref{eq:lower-bound-chi} is greater or equal to 1. Therefore, we have that:
\begin{align}
    \E \VIPS{\pi}^2 \geq \E\VsIPS{\pi}^2 + \frac{\sigma^2}{n} \E_{\Prob(x)}\frac{\pi^2\left(\Phi^0(x) \mid x\right)}{\rho\left(\Phi^0(x) \mid x\right)}.
    \label{eq:lower-bound-2}
\end{align}
Then, using the fact that the $\VsIPS{\pi}$ always underestimates $V(\pi)$ and the upper bound in Equation~\eqref{eq:bias_bound}, we get:
\begin{align}
    \E\VsIPS{\pi} - V(\pi) & \geq -\highr \E_{\Prob(x)} \pi\left(y \in \Psi(x) \cap \Phi^0(x) \mid x\right)
\end{align}
Put together, we recover the bound:
\begin{align}
    \MSE{\VIPS{\pi}} &\geq \MSE{\VsIPS{\pi}} + \frac{\sigma^2}{n} \E_{\Prob(x)}\frac{\pi^2\left(\Phi^0(x) \mid x\right)}{\rho\left(\Phi^0(x) \mid x\right)} -2\highr^2 \E_{\Prob(x)} \pi\left(\Psi(x) \cap \Phi^0(x) \mid x\right)
\end{align}
\end{proof}

\section{Data splitting and logging policy construction}

\label{app:logging}
We split the datasets into train / test according to the AttentionXML manuscript (also in Table 1 of our manuscript). All results are reported on the test dataset. We first run AttentionXML on a fraction $\alpha$ of the training dataset. The resulting probabilistic latent tree provides us with non-normalized marginal probabilities $\hat{p}(y \mid x)$ for each label $y$ and instance $x$. For all experiments, we keep only those probabilities for the top 100 actions for each instance. Then, we may control the randomness over the logging policy in two complementary ways. First, we may add an iid centered Gumbel random variable $g$ (with scaling parameter $\beta$) in order to perturb the rank of the actions:
\begin{align}
    E(x, y) = \log \hat{p}(y \mid x) + g.
\end{align}
We use this step only for the robustness analysis. Second, we may scale this energy by a temperature parameter $T$ to get the re-normalized probability distribution:
\begin{align}
    \rho(y \mid x) = \frac{e^\frac{E(x, y)}{T}}{\sum_{y' \in \mathcal{Y}}e^\frac{E(x, y')}{T}},
\end{align}
with the convention that $E(x, y) = -\infty$ if $y$ is not in the top 100 action for $x$ with respect to $\hat{p}(y \mid x)$. We report all the values of $\alpha, \beta$ and $T$ in Table~\ref{tab:logging_param}.

Once the logging policy constructed, we train all the bandit feedback algorithms on all of the training set instances $x$, for which we attribute actions $y$ sampled without replacement from $\rho$ instead of the optimal actions $y^*$. 

\begin{table}[ht]
 \caption{Parameters for logging policies from XMC datasets.}
\vspace{-0.1in}
\begin{center}
\begin{small}
\label{tab:logging_param}
\begin{tabular}{@{}crrrrrrrrrrrrrrrr@{}}
\toprule
Dataset & Variant & $\alpha$ & $\beta$ & $T$  \\
\midrule
EUR-Lex~\cite{mencia2008efficient} & A & 0.2 & 0 & 2 \\
EUR-Lex~\cite{mencia2008efficient} & B & 0.2 & 1.5 & 5 \\
EUR-Lex~\cite{mencia2008efficient} & C & 0.2 & 4 & 18 \\
Wiki10-31K~\cite{zubiaga2012enhancing} &  & 0.1 & 0 & 1 \\
Amazon-670K~\cite{mcauley2013hidden} &  & 0.2 & 0 & 2 \\
\bottomrule
\end{tabular}
\end{small}
\end{center}
\end{table}

\section{Treatment of tail labels in POXM}
\label{app:tail}

In this section, we explain how to adapt the weighing strategy from~\cite{jain2016extreme} to the case of POXM. We refer to this modified algorithm as wPOXM. 

For developing the intuition of this algorithm, we simply detail the case $\ell=1$. In this setting, the value of a policy $\pi$ can be defined as:
\begin{align}
    V(\pi) = \E_{\mathcal{P}(x)} \E_{\pi(y \mid x)} \E\left[r \mid x, y\right].
\end{align}
In the classification setting, $r$ is a binary random variable indicating whether the label $y$ is relevant for context $x$. Now, we extend this setting by assuming that the label set we observe is incomplete (the main assumption of \cite{jain2016extreme}) and that there exists an unobserved random variable $y^*$ that denotes the complete label set. 

In this setting, we would like to maximize the reward over the complete label set distribution:
\begin{align}
    V(\pi) = \E_{\mathcal{P}(x)} \E_{\pi(y^* \mid x)} \E\left[r \mid x, y^*\right].
\end{align}
However, we only observe the data on the logging policy, so we must reweigh the samples accordingly:
\begin{align}
    V(\pi) = \E_{\mathcal{P}(x)} \E_{\rho(y \mid x)} \frac{\pi(y \mid x)}{p_y\rho(y \mid x)} \E\left[r \mid x, y\right],
\end{align}
where $p_y$ are the propensity weights estimated from~\cite{jain2016extreme}.

We therefore implemented a simple extension of POXM, named wPOXM, that follows this reweighing scheme. We benchmarked this approach against POXM in the EUR-LeX dataset and report the results in Table~\ref{tab:wPOXM}.

\begin{table}[h]
\caption{Performance for wPOXM relative to POXM on the EUR-Lex dataset.\label{tab:wPOXM}}
\begin{center}
\vspace{-0.1in}
\begin{tabular}{@{}lcccccccccccccccccccccccccccccc@{}}
		\toprule
		\textbf{R@3} & \textbf{R@5}&\textbf{nDCR@3}&\textbf{nDCR@5}&\textbf{PSR@3} & \textbf{PSR@5} & \textbf{PSnDCR@3} & \textbf{PSnDCR@5} \\[0.1cm]
		-5.38\% & -5.00\% & -4.53\% & -4.38\% & +4.77\% & +4.28\% & +5.73\% & +6.51\% \\ 
		\hline 
\end{tabular}
\end{center}
\end{table}

As one can see from this table, weighing by the inverse propensity score is effective, as all the weighted metrics are significantly improved. We note a lower performance for the other rewards metrics, explain by the fact that head labels get less often chosen by the policy. All in all, those results show that POXM's paradigm is flexible and can be extended to take into account the tail labels, an important aspect of eXtreme multi-label classification.
\end{document}